\documentclass[11pt]{article}
\usepackage{gal-basic}

\title{Multi-group Agnostic PAC Learnability}
\author{Guy N.\ Rothblum \thanks{This project has received funding from the European Research Council (ERC) under the European Union’s Horizon 2020 research and innovation programme (grant agreement No. 819702), from the Israel Science Foundation (grant number 5219/17), from the U.S.-Israel Binational Science Foundation (grant 2018102), and from the Simons Foundation Collaboration on the Theory of Algorithmic Fairness. Part of this work was done while the author was visiting Microsoft Research.}
\\Weizmann Institute of Science\\\texttt{rothblum@alum.mit.edu}
\and
Gal Yona \thanks{This project has received funding from the European Research Council (ERC) under the European Union’s Horizon 2020 research and innovation programme (grant agreement No. 819702), from the Israel Science Foundation (grant number 5219/17) and from the Simons Foundation Collaboration on the Theory of Algorithmic Fairness. This research was also partially supported by the Israeli Council for Higher Education (CHE) via the Weizmann Data Science Research Center and by a research grant from Madame Olga Klein–Astrachan.}
\\Weizmann Institute of Science\\\texttt{gal.yona@weizmann.ac.il}}

\begin{document}

\maketitle

\begin{abstract}
An agnostic PAC learning algorithm finds a predictor that is competitive with the best predictor in a benchmark hypothesis class, where competitiveness is measured with respect to a given loss function. 
However, its predictions might be quite sub-optimal for structured subgroups of individuals, such as protected demographic groups.
Motivated by such fairness concerns, we study ``multi-group agnostic PAC learnability'': fixing a measure of loss, a benchmark class $\H$ and a (potentially) rich collection of subgroups $\G$, the objective is 
 to learn a single predictor such that the loss experienced by every group $g \in \G$ is not much larger than the best possible loss for this group within $\H$.   Under natural conditions, we provide a characterization of the loss functions for which such a predictor is guaranteed to exist. For any such loss function we construct a learning algorithm whose sample complexity is logarithmic in the size of the collection $\G$. Our results unify and extend previous positive and negative results from the multi-group fairness literature, which applied for specific loss functions.
\end{abstract}

\section{Introduction}
\label{sec:intro}
Machine learning tools are used to make and inform increasingly consequential decisions about individuals. Examples range from medical risk prediction to hiring decisions and criminal justice. Automated classification and risk prediction come with benefits, but they also raise substantial societal concerns. One prominent concern is that these algorithms might discriminate against protected and/or disadvantaged groups. In particular, a learned predictor might perform differently on a protected subgroup compared to the general population. The growing literature on algorithmic fairness has studied different concerns. Many works aim to ensure parity or balance between demographic groups, e.g. similar rates of positive predictions or similar false positive or false negative rates \cite{hardt2016equality,kleinberg2016inherent}. Other works consider accuracy guarantees, such as calibration \cite{dawid1982well}, for protected groups. Protections at the level of a single group might be too weak \cite{dwork2012fairness}, and recent works have studied extending these notions to the setting of multiple overlapping groups \cite{hkrr,kearns2018preventing}.

In this work, we focus on the setting of (supervised) agnostic learning \cite{kearns1994toward}. Given an i.i.d. training set of labeled data, the goal is to learn a predictor $h$ that performs well on the underlying distribution. Performance is measured by a loss function, and with respect to a fixed class $\H$: the loss incurred by the predictor $h$ should be competitive with the best predictor in $\H$. To capture a wide variety of settings, we aim to be quite general in our treatment of different loss functions. With fairness in mind, the agnostic learning paradigm raises a fundamental  concern: since the predictor's loss is measured over the entire underlying distribution, it might not reflect the predictor's performance on sub-populations such as protected demographic groups. Indeed, it has been demonstrated that standard machine learning tools, when applied to standard data sets, produce predictors whose performance on protected demographic groups is quite poor \cite{BuolamwiniG18}.

Motivated by these concerns, we study {\em multi-group} agnostic learning. For a rich collection $\G$ of (potentially) overlapping groups, our goal is to learn a single predictor $h$, such that the loss experienced by every group $g \in \G$ (when classified by $h$) is not much larger than the loss of the best predictor {\em for that group} in the class $\H$. We emphasize that this should hold for all groups in $\G$ simultaneously.

To see how this objective is different from the usual agnostic PAC, consider the simple example in which $\H$ is the class of hyperplanes and we have two subgroups $S,T \subseteq \X$. Suppose that the data is generated such that every group $g$ has a hyperplane $h_g$ that has very low error on it (but that these are different, so e.g. $h_T$ has large loss on $S$ and vice versa). This means that there is no classifier $h \in \H$ that perfectly labels the data. If $S$ is small compared to $T$, then the agnostic learning objective could be satisfied by $h_T$, the optimal classifier for $T$. For \emph{multi-group} agnostic PAC, the fact that there is some other classifier in $\H$ that perfectly labels $S$ serves to disqualify $h_T$ (more generally, it could be the case that no $h \in \H$ will be multi-PAC). This also highlights that the multi-group objective becomes challenging when the groups in question are intersecting (if the groups are disjoint, we can combine the optimal classifiers for each group \cite{dwork2017decoupled}). 

A recent work by Blum and Lykouris \cite{blum2019advancing} studied this question in an online setting with sequential predictions. Our focus is on the batch setting. They showed that (for every collection of groups and every benchmark hypothesis class) it is possible to achieve competitive loss for all groups, so long as the loss function is {\em decomposable}: the loss experienced by each group is an average of losses experienced by its members (this result also applies to the batch setting). On the other hand, they showed a loss function (the average of false negative and false positive rates), for which the objective is infeasible even in the batch setting. Since this loss is non-decomposable (the false positive rate of a classifier depends on the negative rate, which is a property of the entire distribution in question), one might conjecture that the multi-group objective is only possible for decomposable losses. However, this conjecture is false: For example, calibration \cite{dawid1982well} is a non-decomposable loss that is compatible with the multi-group objective. Indeed, \cite{hkrr} propose an algorithm that guarantees predictions are calibrated across groups (in other words, the calibration loss for each $g \in \G$ is close to zero). In particular, the output of this algorithm would satisfy the multi-group agnostic PAC requirement (with respect to the calibration loss, and for every hypothesis class).

\subsection{Our contributions}

Motivated by these observations, in this work we formalize two main questions:

\begin{enumerate}
    \item We define a loss function as (multi-group) \emph{compatible} if it is ``appropriate'' to use with our objective, in the sense that for every hypothesis class $\H$ and collection of groups $\G$, there exists a hypothesis $h$ that is competitive with $\H$ for every group in $\G$. What makes a loss function compatible? Previous works provide several positive and negative results, but there was no clear characterization of compatibility.
    \item Is it always the case that for such ``compatible'' losses, a multi-group predictor can also be found using a finite number of samples? In other words, is there a separation between multi-group {\em compatibility} and multi-group {\em learnability}? 
\end{enumerate}

Our main technical contributions answer both questions:

\begin{enumerate}
    \item We prove a partial information-theoretic characterization of the compatible loss functions.
    
    \item For any such loss function that also satisfies a natural uniform convergence property, we show an algorithm that, for any specified finite collection $\G$ and finite hypothesis class $\H$, learns a multi-group agnostic predictor from labeled data. The sample complexity is logarithmic in the sizes of $\G$ and $\H$. Our algorithm is derived by a reduction to {\em outcome indistinguishability} (OI), a learning objective recently introduced by Dwork {\em et al.} \cite{dwork2020outcome}, drawing a new connection between OI and agnostic learning. This shows that (under minimal assumptions on the loss function), multi-group compatibility implies multi-group learnability.

\end{enumerate}

In slightly more detail, 
we characterize the compatible loss functions assuming an additional {\em unambiguity} property (we refer to the characterization as ``partial'' because of this assumption): we assume that once we fix a single individual and specify the distribution of their label, there is a unique prediction that minimizes the loss for that individual. We view  this as a very natural assumption on the loss function, see the discussion following Definition \ref{def:unambiguity}. We show that if a loss function is compatible and unambiguous, then for each individual, the optimal prediction can be obtained by a fixed ``local'' function $f$, whose output only depends on the features and on the marginal distribution of that individual's label. The point is that the function $f$ doesn't depend on the global distribution, but the predictions that it specifies still minimize the loss for that distribution. We call loss functions that satisfy this property {\em $f$-proper}, and we show that (under the unambiguity assumption) being \nice is {\em equivalent} to multi-group agnostic PAC compatibility. 

We then construct a universal multi-group agnostic PAC learning algorithm for any $f$-proper loss function that also satisfies a minimal uniform convergence property (this is necessary for finite sample complexity). The learning algorithm works for any specified finite hypothesis class $\H$ and any finite collection of groups $\G$, and its sample complexity is logarithmic in $|\H|$ and  $|\G|$. The algorithm is obtained via a reduction from multi-group agnostic PAC learning to the recently studied task of outcome indistinguishable learning \cite{dwork2020outcome}. Beyond unifying previous work in \emph{multi-group fairness}, this result can be thought of as a multi-group analogue for the known result that every PAC learnable class is learnable via empirical risk minimization.

\subsection{Related work}

The literature on algorithmic fairness is broad and growing. Most pertinent to our work is the literature on fairness notions that aim to stake a middle-ground between the strong, individual-level semantics of \emph{individual fairness} notion \cite{dwork2012fairness} and the weaker but useful \emph{group fairness} notions \cite{hardt2016equality, kleinberg2016inherent}. We broadly refer to such notions as \emph{multi-group} notions. Some of these works are geared towards guaranteeing parity: equalizing some statistic across the (possibly exponentially-many) subgroups in $\G$. For example, \cite{kearns2018preventing} study multi-group versions of both Equality of Opportunity and Demographic Parity notions \cite{hardt2016equality}. We note however, that as the collection $\G$ becomes richer and richer, it might be the case multi-group parity can only be obtained via trivial predictions or otherwise undesirable behaviour, such as intentionally making worse predictions on some groups to achieve parity \cite{chen2018my}.
A different line of works consider guarantees that are not inherently in conflict with accuracy, e.g. because the Bayes optimal predictor always satisfies multi-group fairness (regardless of $\G$). Notable examples include multi-calibration \cite{hkrr} and multi-accuracy \cite{kim2019multiaccuracy}, with subsequent works studying 
extensions in ranking \cite{dwork2019learning}, regression \cite{jung2020moment} and online learning \cite{gupta2021online}. As discussed above,  Blum and Lykouris \cite{blum2019advancing} study multi-group agnostic PAC learning (which, depending on the loss function, is often similarly aligned with accuracy) in the online setting. 
In the batch setting, their approach for learning a  multi-PAC predictor is via the mixture of experts approach \cite{jacobs1991adaptive}, where  the ``experts'' in question are a set of $\card{\G}$ predictors, with each $h_i$ being the optimal predictor for group $g_i \in \G$ in the class $\H$. However, this approach fails to produce a multi-PAC predictor when the loss is non-decomposable, such as calibration (this happens even though, by \cite{hkrr}, such multi-fair predictors exist and can be found in sample complexity that scales with $\log\cG$).

\paragraph{Discussion.} The strength of the multi-group agnostic PAC guarantee depends on  both the choice of subgroups $\G$ and the choice of benchmark class $\H$. As in the majority of the multi-group fairness literature \cite{hkrr, dwork2019learning, kim2019multiaccuracy}, the guarantee might be weak if the groups are not sufficiently descriptive (either due to the features or the choice of the class $\G$ itself). In Similarly, the benchmark class $\H$ could be insufficient (again, due to the features or the choice of the functions in $\H$). Both of these choices could be manipulated by a decision-maker with an intention to discriminate against a certain subgroup, e.g. by deleting features that will allow this group to be defined within $\G$ or by intentionally choosing $\H$ to only contain bad classifiers for this group, if it is in $\G$.  In either case, this highlights that the multi-group agnostic PAC guarantee should only be interpreted w.r.t the choice of $\G$ and $\H$, and that these choices should be scrutinized separately.

\paragraph{Organization.} The rest of this manuscript is organizes as follows. Section \ref{sec:prelims} covers the background on our general formulation of loss functions, as well as the learning paradigms (Agnostic PAC and Outcome Indistinguishability). In Section \ref{sec:multi} we formalize the notions of compatibility and learnability in the multi-group sense. Section \ref{sec:equiv} contains our main result, the equivalence between the two notions. Full proofs are deferred to the appendices.

\section{Preliminaries}
\label{sec:prelims}

\textbf{Setup and notation.} We consider binary classification problems, where $\X \subseteq \R^d$ denotes a feature space and $Y = \set{0,1}$ the target labels. As a general convention, we use $\D$ to denote distributions over $\X \times Y$ and $\D_X$ for  the marginal distribution of $\D$ over $\X$. The support of a distribution (w.r.t $\X$) is $\supp(\D) = \supp(\D_X)$. 
We will sometimes be interested in distributions over $\X \times Y$ for which $\card{\supp(\D)} =1$, i.e. distributions that are supported on a single element $x \in \X$. 
We refer to these as ``singleton distributions''. 
For a distribution $\D$ and an element $x \in \X$, we use $\D_x$  to denote the singleton distribution on $\X \times \Y$ obtained by restricting  $\D$ to $X = x$. A predictor is a mapping $h: \X \to [0,1]$, where $h(x)$ is an estimate for the probability that the label of $x$ is 1. We will sometimes consider the special case of classifiers (binary predictors), whose range is $\set{0,1}$. A hypothesis class is a collection of predictors, and is denoted by $\H$. A subgroup is some subset of $\X$. A collection of subgroups is denoted by $\G$. For this work we generally assume $\H$ and $\G$ are finite (but possibly exponentially large in $d$). For a distribution $\D$ and a group $g \subseteq \X$, we use $\D_g$ to denote the distribution on $\X \times Y$ obtained by restricting $\D$ to $x \in g$.

\subsection{General loss functions}
\label{sec:losses}
A loss function $L$ is some mapping from a distribution $\D$ and a predictor $h$ (technically, its' restriction to $\D$) to $[0,1]$. We are typically interested in how $L$ changes as we keep the first argument (the distribution) fixed, and only change the second argument (the predictor in question). We thus typically use $L_\D(h)$ to denote the loss of $h$ w.r.t. a distribution $\D$. For a sample $S = \set{(x_i, y_i)}_{i=1}^m$ we use $L_S(h)$ to denote the empirical loss, calculated as $L_{\hat{\D}}(h)$, where $\hat{\D}$ is the empirical distribution defined by the sample $S$. 
Note that this setup is extremely general, and assumes nothing about the loss (except that it is bounded and can't depend on what happens outside $\D$). In machine learning it is common to consider more structured losses, in which $L_\D(h)$ is the expected loss of $h$ on a random example drawn according to $\D$. We refer to such structured losses as \emph{decomposable} losses. 

\begin{definition}[Decomposable losses]
 A loss function $L$ is \emph{decomposable} if there exists a function $\ell: X \times Y \times [0,1] \to [0,1]$ such that for every distribution $\D$ and classifier $h$, 
$L_\D(h) = \E_{(x,y) \sim \D}[\ell(x, y, h(x))]$.
\end{definition}

For example, for binary classifiers a standard decomposable loss is the 0-1 loss, in which  $\ell(x,y,h(x))=\textbf{1}[h(x)\neq y]$. For predictors, an example of a standard decomposable loss is the squared loss, in which $\ell(x,y,h(x))=(h(x)-y)^2$. 

\paragraph{Beyond decomposable losses.} While decomposable losses are standard and common, there are many loss functions of interest that don't have this form -- especially in the literature on algorithmic fairness. For this reason, we focus on a general notion of loss functions (which does not explicitly assume losses are  decomposable) in our exploration of multi-group agnostic PAC learning. 
Notable examples of such losses, as used in the algorithmic fairness literature, include the following notions: 

\begin{itemize}

\item \textbf{Calibration} \cite{hkrr, kleinberg2016inherent, chouldechova2017fair, shabat2020sample}: A predictor is \emph{calibrated} if for every value $v \in [0,1]$, conditioned on $p(x)=v$, the true expectation of the label is close to $v$. Intuitively, this means that the outputs of the predictor can reasonably be thought of as probabilities, and hence is a fundamental requirement in the literature on forecasting \cite{dawid1982well, foster1998asymptotic}. 
For example, in weather prediction, calibrated forecasts ensure that, out of all the days on which the forecaster predicted say 0.8, it really rained on 80\% of them. Calibration loss measures the extent to which a predictor is miscalibrated; e.g., the \emph{expected calibration error} \cite{kumar2018trainable} measures the deviation from calibration w.r.t a value $v$ drawn from the distribution induced by the prediction in questions. This loss is not decomposable because it is a global function of the predictions, not a property of the prediction for a single $x \in \X$.
    
\item \textbf{One-sided error rates} \cite{hardt2016equality, chouldechova2017fair, blum2019advancing, blum2019recovering, kearns2018preventing}: The \emph{false positive rate} (similarly, false negative rate) measures the probability of a random example being labeled as $h(x)=1$, conditioned on the true label being $y=0$. This isn't a decomposable loss because the exact contribution of a single misclassification depends on the frequency of the negative labels, which is a global property.

\item \textbf{Individual fairness }\cite{dwork2012fairness, yona2018probably}: Individual fairness (IF) requires that individuals that are considered similar w.r.t the task at hand are treated similarly by the classifier. Similarity is specified by a metric $d: \X \times \X \to [0,1]$. If $d(x,x')\approx 0$ ($x,x'$ are similar), then it should be the case that $h(x) \approx h(x')$. An IF loss may quantify the expected violation of this requirement, e.g. over a draw of a pair $(x,x')$ i.i.d from $\D$. This objective isn't decomposable because the requirement w.r.t a single $x \in \X$ depends on the extent to which there are other similar $x'$ in $\supp(\D)$. 

 \end{itemize}
 
We note that both the latter two losses (false positive rate and Individual fairness) are typically not interesting on their own -- e.g. because the constant classifier $h \equiv 0$ minimizes them. Typically we are interested in these losses as either an additional constraint on an existing objective, or paired with an additional loss (e.g. IF + accuracy based loss, or false positive rate + false negative rate).

  In the rest of this section we continue to specify two minimal conditions for the losses we consider. We will use the notation $\Loss$ for the class of losses that satisfy both conditions. 
The first condition is \emph{unambiguity}. It guarantees that  distributions over a single example $x$ have a unique loss minimizer.

\begin{definition}[Unambiguity]
\label{def:unambiguity}
A loss $L$ is unambiguous if for every singleton distribution $\D_x$ over $\X \times Y$, there is a unique prediction $h(x)$ that minimizes the loss.  That is, $\card{\arg\min_{h(x)\in [0,1]}L_{\D_x}(h)} = 1$. 
\end{definition}

Standard decomposable losses satisfy this condition because the function $\ell(h(x),y)$ typically has a unique minimum w.r.t its first argument. For example when $\ell$ corresponds to the squared loss the optimal labeling is $h(x)=\E \sbr{y\vert x}$ and when $\ell$ corresponds to the expected 0-1 loss it is $h(x)=1\sbr{\E \sbr{y\vert x}\geq 0.5}$. 
With respect to the fairness-motivated losses mentioned above, unambiguity fails w.r.t  individual fairness and the one-sided error rates (exactly because they can both be minimized by the constant all-zeroes classifier, regardless of the true distribution). But as mentioned above, these losses are rarely interesting on their own, and combined losses (individual fairness with an unambiguous loss, or an average of false positive rate and false negative rate) are unambiguous. In other words, all losses that are of practical interest to us are unambiguous. 

The second condition we will require is that the empirical risk $L_S(\cdot)$ can really be used to approximate the true risk $L_\D(\cdot)$, for a sufficiently large sample $S$. To build up to this notion we first recall the standard definition of  uniform convergence for hypotheses classes.

\begin{definition}[Uniform Convergence for hypotheses classes]
 We say that a hypothesis class $\H$ has
the uniform convergence property (w.r.t. a domain $X\times Y$ and a loss function $L$) if
there exists a function $m^{UC}_\H : (0, 1)^2 \to \N$ such that for every $\eps, \delta \in (0, 1)$ and
for every probability distribution $\D$ over $X\times Y$, if $S$ is a sample of $m \geq m^{UC}_\H (\eps, \delta)$
examples drawn i.i.d. according to $\D$, then, with probability of at least $1- \delta$, $\forall h \in \H: \quad \card{L_S(h) - L_D(h)} \leq \eps$.
\end{definition}

In our context, we will be interested in uniform convergence as a property of the \emph{loss function}. We will say that a loss $L$ has  uniform convergence (w.r.t finite classes) with sample complexity $m^{UC}_L: (0,1)^2 \times \N \to \N$ if  every finite class $\H$
has the uniform convergence property w.r.t $L$ with sample complexity $m^{UC}_\H(\eps, \delta) \leq m^{UC}_L(\eps, \delta, \cH).$ Specifically, we will be interested in losses that have the uniform convergence property with sample complexity that depends polynomially on $1/\eps, 1/\delta$ and $\log\card{\H}$. This gives rise to the following definition:

\begin{definition}[Uniform convergence for loss functions]
\label{def:uniform-convergence}
A loss $L$ has the uniform convergence property (w.r.t finite classes) with sample complexity  $m^{UC}_L: (0,1)^2 \times \N$ if there exists a polynomial $f: \R^3 \to \N$ such that for every $\eps, \delta \in (0,1)$ and $k \in \N$, 
 \begin{equation*}
     m^{UC}_L(\eps, \delta, k) \triangleq \max_{\H: \, \cH = k}m^{UC}_\H(\eps, \delta) \leq f(1/\eps, 1/\delta, \log(k))
 \end{equation*}

\end{definition}

The uniform convergence property is satisfied by any decomposable loss function. This follows by a combination of Heoffding's bound (for a single $h$) and a union bound to get a simultaneous guarantee for every $h \in \H$.  Out of the fairness-motivated losses we discussed above, only the loss $L_\D(h) = a\cdot \mathtt{FPR}_\D(h) + b\cdot \mathtt{FNR}_\D(h)$ doesn't have uniform convergence, as we prove in Appendix \ref{appendix:fpr}. For calibration, 
uniform convergence follows as a special case of the bounds in \cite{shabat2020sample}; for individual fairness, the argument is similar to the standard argument for decomposable losses, only this time the concentration argument is w.r.t pairs of samples from $\D$ \cite{yona2018probably}.

To summarize, we have defined a collection of ``reasonable`` losses $\Loss$ as all loss functions that are both unambiguous and have the uniform convergence property. We have argued that of the loss functions we discussed, this only rules out the one-sided error rate loss; see Figure \ref{fig:losses} for an overview.

\begin{figure}
    \centering
    \begin{tabular}{cccc}
\multicolumn{3}{c}{} \\
\toprule
\text{loss} &   \text{notation} & \text{in}\, $\Loss$?   \\
\midrule
decomposable   &    $L^{\mathtt{\ell}}_\D(h)$ &    \cmark \\
calibration &    $L^{\mathtt{C}}_\D(h)$ &    \cmark \\
IF + decomposable     &   $a\cdot L^{\mathtt{IF}}_\D(h) +  b\cdot L^{\mathtt{\ell}}_\D(h)$ & 
\cmark \\
error rates    &    $a \cdot L^{\mathtt{FPR}}_\D(h) + b \cdot L^{\mathtt{FNR}}_\D(h)$ & \xmark  \\

\bottomrule
\end{tabular}
    \caption{Summary: loss functions and the set $\Loss$ (all losses that are unambiguous and have the uniform convergence property.)}
    \label{fig:losses}
\end{figure}

\subsection{Learning paradigms}

In this work we draw a connection between an extension of the classic notion of agnostic PAC learnability in the presence of multiple groups and the recently introduced complexity-theoretic inspired framework of Outcome Indistinguishability \cite{dwork2020outcome}, which builds on previous work in algorithmic fairness \cite{hkrr}. In this section we review both of these learning paradigms, starting with agnostic PAC learnability. For consistency, we give the definition w.r.t finite classes.

\begin{definition}[Agnostic-PAC learnability]
\label{def:agnostic-pac}
A hypothesis class $\H$ is agnostic PAC learnable with respect to a set $X \times Y$ and a
loss function $L$ if there exist a function $m_\H : (0, 1)^2 \to \N$
and a learning algorithm with the following property: For every $\eps, \delta \in (0,1)$
and for every distribution $\D$ over $\X \times Y$, when running the learning algorithm on
$m \geq m_\H(\eps, \delta)$ i.i.d. examples generated by $\D$, the algorithm returns $h$ 
such that, with probability of at least $1-\delta$ (over the choice of the $m$ training
examples),
\begin{equation}
\label{eqn:pac}
    L_\D(h) \leq L_\D(\H) + \eps
\end{equation}
Additionally, the sample complexity  $m_\H$ should be polynomial in $1/\eps, 1/\delta$ and in $\log(\cH)$.
\end{definition}

\paragraph{Outcome indistinguishability.} A predictor $\pt: \X \to [0,1]$ can be viewed as providing a generative model for outcomes, where for $x \in \X$ a binary outcome is sampled as  $y \sim \text{Ber}(\pt(x))$. Given a distribution $\D$ and a predictor $\pt$, we define the ``modeled distribution'' $\tilde{\D} = \D(\pt)$ as the distribution over $\X \times Y$ obtained by first drawing $x$ according to the marginal distribution of $\D$ on $\X$ and then labeling it as $y \sim \text{Ber}(\pt(x))$. The framework of 
Outcome Indistinguishability (OI) aims to guarantee that outcomes produced by $\pt$ are indistinguishable from the true outcomes under $\D$. This guarantee is formulated with respect to a fixed class $\A$ of distinguishers, and the requirement is that every distinguisher $A \in \A$ behaves similarly on samples from the real distribution $\D$ and on samples from the ``modeled distribution'' $\tilde{\D}$. 

As discussed and studied in \cite{dwork2020outcome}, there are several different ways to instantiate this framework based on the power of the distinguishers. In particular, two axes of variation are (i) the input to the distinguisher (single sample vs multi-sample) and (ii) the access level it receives to the predictor $\pt$ in question. In this work, we focus on multi-sample Sample-Access OI distinguishers, where the inputs to  each distinguisher $A \in \A^k$ are $k-$tuples of the form $\set{(x_i, y_i, \pt_i)}_{i=1}^k$, where for every $i \in [k]$, $(x_i, y_i)$ is sampled from either $\D$ or $\tilde{\D}$:

\begin{definition}[Outcome Indistinguishability \cite{dwork2020outcome}]
Fix a distribution $\D$, a  collection of distinguishers $\A^k$ and $\tau > 0$. A predictor $\pt: \X \to [0,1]$ satisfies $(\A^k, \tau)$-OI if for every $A \in \A^k$, 
 \begin{multline*}
   \card{\Pr_{\{(x_{i},y_{i})\}_{i=1}^{k}\sim \D^{k}}[A(\{(x_{i},y_{i},\tilde{p}(x_{i})\}_{i=1}^{k})=1]-  \Pr_{\{(x_{i},y_{i})\}_{i=1}^{k}\sim\tilde{\D}^{k}}[A(\{(x_{i},y_{i},\tilde{p}(x_{i})\}_{i=1}^{k})=1]} \leq\tau
 \end{multline*}
\end{definition}

Much like the definition of (regular) PAC learning, we say that an algorithm \emph{learns} multi-sample OI if upon receiving sufficiently many i.i.d samples from the distribution, 
 it is guaranteed to return a predictor  that is ``probably approximately'' OI.

\begin{definition}[OI learnability]
\label{def:oi}
A family of $k$-sample distinguishers $\A^k$ is multi-sample OI learnable  with respect to  $\X \times Y$ if
there exists a function $m_\A: (0,1)^2 \to \N$ and a learning algorithm with the following property: For every $\tau, \eta \in (0,1)$ and for every distribution $\D$, when running the learning algorithm on $m \geq m_\A(\tau, \eta)$ i.i.d example generated by $\D$, the algorithm returns $h$ such that with probability at least $1-\eta$ (over the choice of the $m$ training examples), $h$ satisfies $(\tau, \A^k)$-OI.

Additionally, the sample complexity $m_\A$ should be polynomial in $1/\tau, 1/\eta$ and in $\log(\A^k)$.
\end{definition}

Dwork \emph{et al.} \cite{dwork2020outcome} showed that every  $k-$sample distinguisher class $\A^k$ is OI-learnable. In our analysis 
 we use the following theorem, which bounds the sample complexity of this algorithm and follows from \cite{dwork2020outcome}.

\begin{theorem}[from ~\cite{dwork2020outcome}]
\label{theorem:oi}
Fix $\A^k$. There exists an algorithm $\mathtt{OI}_{\A^k}$ that satisfies the requirement of Definition (\ref{def:oi}), and whose sample complexity is $O\br{\frac{k\cdot\br{\log|A^{k}|/\eta}}{\tau^4}}$.
\end{theorem}

\section{Agnostic PAC  with multiple groups}
\label{sec:multi}

The objective of agnostic PAC learning is to output a predictor $h$ that satisfies $L_\D(h) \lesssim L_\D(\H)$ (Equation \ref{eqn:pac}). Putting aside questions of computational complexity and sample complexity, this objective is itself always feasible. In other words, regardless of the loss in question, it can always be obtained using a learner who has full knowledge of the distribution $\D$ and is not constrained in runtime. From an information-theoretic perspective, this makes the interesting question the question of what can be done with finite samples from $\D$ -- hence the focus on agnostic PAC \emph{learnability}.

A multi-group extension of agnostic PAC asks for a predictor that satisfies the above, but simultaneously for every group $g$ in a collection $\G$: $ L_{\D_g}(h) \lesssim L_{\D_g}(\H)$, where $\D_g$ denotes the restriction of $\D$ to samples from $g$.

When $\G$ consists of intersecting groups, however, it is not immediately clear that this objective is always feasible: it might not be satisfied by \emph{any} predictor $h: \X \to [0,1]$. For a simple (but contrived) example, let $h^0, h^1$ denote the all-zeroes and all-ones predictors, and 
consider a loss $L$ that specifies that $L_{\D_S}(h^0)=0$ and $L_{\D_T}(h^1)=0$ (and for any other classifier $h$, the loss of every distribution is always 1). Then the multi-group objective w.r.t $\G= \set{S,T}$ requires that we label the intersection $S \cap T$ as both 1 and 0, which is impossible. 
The following lemma proves that this can be the case even for seemingly natural losses, that are in the class $\Loss$ defined in Section \ref{sec:losses}.

\begin{proposition}
\label{prop:impossibility}
There exists $L \in \Loss$ for which some problem instances (corresponding to a distribution $\D$, class $\H$, subgroups $\G$ and $\eps >0$) don't admit a multi-group agnostic PAC solution. In other words, there is no classifier $h$ for which $ L_{\D_g}(h) \lesssim L_{\D_g}(\H)$ for every group $g \in \G$.
\end{proposition}

The loss in question is a weighted combination of an individual fairness loss with an accuracy loss. We construct two intersecting groups $S$ and $T$ and a similarity metric $d$ that specifies $d(x,x')=0$ if and only if $x \in S-T$ and $x' \in S \cap T$. However, the true labels of these individuals are different: for $x \in S -T$ the label is $y=0$ but for $x\in S\cap T$ the label is $y=1$. This creates a situation in which group $T$ is optimizing strictly for accuracy (and so it wants the intersection to be labeled as $\hat{y} = 1$) whereas for the group $S$ the dominating term is IF (and so it wants everyone, and in particular the intersection, to be labeled as $\hat{y} =0$). Thus, any classifier that is simultaneously multi-PAC w.r.t both $T$ and $S$ must label the intersection as both $0$ and $1$, which is impossible. See Appendix \ref{appendix:acc-if} for the full proof.

We note that \cite{blum2019advancing} already demonstrated that this impossibility occurs in the batch setting. This motivated them to focus on the 0-1 loss. However, their counter example is for the error-rate type loss we discussed earlier, which is a fixed combination of the false positive rate and false negative rate of a classifier. As noted in Section \ref{sec:losses}, this loss doesn't have the uniform convergence property and is therefore not in $\Loss$. Proposition \ref{prop:impossibility} clarifies that  uniform convergence is not the issue, and there are natural (and otherwise reasonable) losses
that are not ``appropriate'' to use with the multi-group objective (in the  sense that even an approximate solution is not guaranteed to exist). 

In light of this discussion, we proceed to explicitly separate the question of \emph{feasibility} (whether it's always the case that some multi-group solution is guaranteed to exist) from \emph{learnability} 
 (whether we can find it with a ``reasonable'' number of samples). Formally, we define two notions -- compatibility and learnability -- that formalize these concepts. Importantly, both are properties of the loss function in question, taking a universal quantifier over the  class $\H$ and the groups $\G$ (and the other aspects of the problem instance).

\begin{definition}[Multi-PAC compatibility]
\label{def:compatibility}
 We say that a loss $L$ is \emph{multi-PAC compatible} if for every distribution $\D$, class $\H$, subgroups $\G$ and $\eps >0$, there exists $h: \X \to [0,1]$ such that for every group $g\in \G$, $     L_{\D_g}(h) \leq L_{\D_g}(\H) + \eps
$.
\end{definition}

Multi-PAC learnability strengthens the above requirement by asking that such a solution can also be found by a learning algorithm whose sample complexity is constrained to depend inverse-polynomially on the parameters in question and logarithmically on the sizes of $\H$ and $\G$.

\begin{definition}[Multi-PAC learnability]
\label{def:learnability}
We say that a loss $L$ is \emph{multi-PAC learnable}  with sample complexity $m^{gPAC}_L: (0,1)^3 \times \N^2 \to \N$ if there exists a  learning algorithm with the following property: For every $\eps, \delta, \gamma \in (0,1)$, for every finite hypothesis class $\H$, for every finite collection of subgroups $G \subseteq 2^\X$
and for every distribution $\D$ over $\X \times Y$, when running the learning algorithm on
$m \geq m^{gPAC}_L(\eps, \delta, \gamma, \cH, \cG)$ i.i.d. examples generated by $\D$, the algorithm returns $h$
such that, with probability at least $1-\delta$ (over the choice of the $m$ training
examples and the coins of the learning algorithm)  $g \in \G_{\gamma}$, $L_{\D_g}(h) \leq L_{\D_g}(\H) + \eps$, 
where $\G_\gamma \subseteq \G$ is the subset of groups whose mass under $\D$ is at least $\gamma$: $\G_\gamma = \set{g \in \G: \Pr_\D[x\in g] \geq \gamma}$.

Additionally, the sample complexity  $m^{gPAC}_L$ should be polynomial in $1/\eps, 1/\delta, 1/\gamma$ and $\log(\cH), \log(\cG)$.

\end{definition}

\section{Compatibility $\iff$ Learnability via OI}
\label{sec:equiv}

In this section we prove our main result: that for loss functions in $\Loss$,  multi-PAC compatibility implies also multi-PAC learnability. 

\begin{theorem}
\label{thm:compatible-is-learnable}
If a loss $L \in \Loss$ is multi-group compatible (Definition \ref{def:compatibility}), then it is also multi-group learnable (Definition \ref{def:learnability}).
\end{theorem}

Towards proving Theorem \ref{thm:compatible-is-learnable}, we introduce an additional property of  loss functions, which we define below.

\begin{definition}[\nice]
\label{def:nice}
For a function $f: \X \times [0,1] \to [0,1]$, we say that a loss $L$ is \emph{\nice} if for every distribution $\D$ on $\X \times Y$, the classifier $h_\D$ given by $h_\D(x)=f(x,\ygivenx)$ minimizes the loss w.r.t $\D$: $h_\D \in \arg\min_h L_\D(h)$.
\end{definition}

Recall that proper losses (or proper scoring functions) are losses that are minimized by the conditional expectation predictor $x \mapsto \ygivenx$ \cite{buja2005loss}. Definition \ref{def:nice} is a weaker requirement that says that for every distribution a minimizer can be obtained as some \emph{local} transformation of this predictor (i.e. that does not depend on the rest of the distribution).

Recalling that we defined $\Loss$ as precisely all loss functions that satisfy both uniform convergence and unambiguity, we see that Theorem \ref{thm:compatible-is-learnable} follows as a direct corollary of the following two lemmas.


\begin{lemma}
\label{lemma:compatible-is-nice}
If $L$ is unambiguous (Definition \ref{def:unambiguity}) and multi-group compatible (Definition \ref{def:compatibility}), then $L$ is \nice (Definition \ref{def:nice}).
\end{lemma}

\begin{lemma}
\label{lemma:nice-is-learnable}
If $L$ is \nice (Definition \ref{def:nice}) and has the uniform convergence property (Definition \ref{def:uniform-convergence}), then $L$ is multi-group learnable (Definition \ref{def:learnability}).
\end{lemma}

We note that the other direction of Lemma \ref{lemma:compatible-is-nice} is also true (if a loss is $f$-proper, it is also multi-group compatible): this follows immediately, since labeling everyone  in $\supp(\D)$ according to $f$ is optimal -- and in particular satisfies the multi-group requirement for every $\H$. This means that the notion of $f-$proper provides a partial characterization of compatibility 
(up to unambiguity).

The full proofs of Lemmas \ref{lemma:compatible-is-nice} and \ref{lemma:nice-is-learnable} can be found in Appendices \ref{appendix:comp-to-proper} and \ref{appendix:proper-to-learnable}, respectively. In the rest of this section we give an overview of both proofs.

\subsection{Overview of Lemma \ref{lemma:compatible-is-nice}}

Let $L \in \Loss$ be a loss function that is multi-group compatible and unambiguous (we do not use the uniform convergence property here). We show that $L$ is also $f$-proper, where the function $f$ maps a singleton distribution (specified by an input $x$ and the conditional probability that the label is 1) to whatever prediction minimizes the loss on that distribution. In more detail: for an input $x \in \X$ and a value $z \in [0,1]$, let $\D_{x,z}$ be the singleton distribution over $ \{x\} \times \{0,1\}$, where the label is drawn by the Bernoulli distribution $\text{Ber}(z)$. Let $h_{x,z}$ be the predictor that minimizes the loss on this distribution, i.e. $h_{x,z} = \arg\min_{h} L_{\D_{x,z}}(h)$, and recall that by unambiguity, the prediction that minimizes the loss is unique, and so the value $h_{x,z}(x)$ is well defined. We take $f(x,z) = h_{x,z}(x)$.

It remains to prove that $L$ is an $f$-proper loss function. Suppose for contradiction that it is not: i.e., there exists a distribution $\D$ for which $h_\D(x)=f(x,\ygivenx)$ does not minimize the loss, i.e. there exists some predictor $h'$ s.t. $L_{\D}(h') < L_{\D}(h_\D)$. We show this contradicts the multi-group compatibility of $L$. To see this, define a collection of sets that includes all the singletons in the support of $\D$, as well as the global set comprised of the entire support of $\D$. By unambiguity, the singletons all ``want'' to be labeled by $h_D$. On the other hand, the global set wants to be labeled by $h'$. Whatever predictor we pick, the loss on either the global set or of one of the singletons will not be optimal. Thus, for the above collection of groups $\G$ (comprised of all singletons plus the global set), and for the hypothesis class $\H$ that includes $h_{\D}$ and $h'$, it is impossible to obtain predictions that are competitive with $\H$ for all groups in $\G$ simultaneously.

\begin{algorithm}[h]
  \caption{$\mathtt{MultiGroup_{L,f}}(\epsilon,\delta,\gamma,\H,\G)$}
  \label{algo:multi-group}
      \begin{algorithmic}[1]
        \STATE {\bfseries Parameters:} loss function $L$, function $f: \X \times [0,1] \to [0,1]$
        
        \vspace{2mm}
        \STATE {\bfseries Input:} accuracy parameter $\epsilon \in (0,1)$, failure probability $\delta \in (0,1)$, minimal subgroups size parameter $\gamma \in (0,1)$, hypothesis class $\H$, collection of subgroups $\G$. 
        \STATE {\bfseries Output:} A classifier $h$ satisfying the $(\eps, \delta)$-multi-group guarantee w.r.t $\H$ and $\G$
        \vspace{2mm}

        \STATE Set $\eps'=\alpha = \eps/4$ and $\delta'=\eta = \tau = \delta/4$.

        \STATE{Set $k_\G = m^{UC}_L\br{\eps', \delta', \cH + 1}$}.
        
        \STATE Set $ k=10\cdot\frac{1}{\gamma}\cdot\log\frac{1}{\delta'}\cdot k_\G $.
        
        \STATE Let $\A = \set{A^{L, f, k}_{ g,h,\alpha} \,\, \vert \,\, g\in \G, h\in \H} $ be a collection of distinguishers, as defined in Algorithm \ref{algo:dist}.

        \STATE Invoke OI as a sub-routine to learn $\tilde{p}\leftarrow\mathtt{OI}(\tau,\eta,\A)$.
        
        \RETURN{$f(\tilde{p})$}
        
      \end{algorithmic}
\end{algorithm}

\begin{algorithm}[h!]
  \caption{$A^{L, f, k}_{ g,h,\alpha}$ (multi-sample Sample-Access OI distinguisher)}
  \label{algo:dist}
      \begin{algorithmic}[1]
        \STATE {\bfseries Parameters:} number of samples $k \in \N$, group $g \subseteq \X$, classifier $h: \X \to [0,1]$, loss function $L$, function $f: \X \times [0,1] \to [0,1]$
        \vspace{2mm}
        \STATE {\bfseries Input:} $\set{(x_i, y_i, p_i)}_{i=1}^k$, where $x_i \in \X$, $y_i \in \set{0,1}$ and $p_i \in [0,1]$
        \STATE {\bfseries Output:} A binary output denoting Accept/Reject
        \vspace{2mm}
        
        \STATE $I_g = \set{i: x_i \in g}$ 
        \STATE $S_g = \set{(x_i, y_i)}_{i \in I_g}$
        
        \STATE Define a predictor $f_g$ as
        \begin{equation}
        \label{eqn:f_g}
        f_g(x) = 
        \begin{cases}
         f(x_i, p_i) & \exists i \in [k] \,\, \text{such that} \,\,  x=x_i \\
         0 & \text{otherwise}
        \end{cases}
        \end{equation}
        
        \IF{$L_{S_g}(f_g) < L_{S_g}(h) + 2\alpha$}
            \RETURN{1}
        \ENDIF
        \RETURN{0}

      \end{algorithmic}
\end{algorithm}

\subsection{Overview of Lemma \ref{lemma:nice-is-learnable}}

Given a loss function $L$ that is $f$-proper and has the uniform convergence property, we want to construct a multi-group agnostic PAC learning algorithm that works for any (finite) hypothesis class $\H$ and (finite) collection of groups $\G$. The algorithm will work by a reduction to the task of OI learning (see above): namely, we construct a collection $\A$ of distinguishers, and show that any predictor $\pt$ that is OI w.r.t this collection can be used to derive a multi-group agnostic predictor $h$. In particular, we show that if $\pt$ is OI (w.r.t $\A$), then $\tilde{h}(x) = f(x,\pt(x))$ is a multi-group agnostic predictor (recall that $f$ is the local transformation for the $f$-proper loss function $L$). The collection of distinguishers depends on the loss function $L$, on the hypothesis class $\H$ and on the collection of groups $\G$. This reduction, together with the OI learning algorithm of Theorem \ref{theorem:oi} (from \cite{dwork2020outcome}), gives a ``universal'' multi-group agnostic learning algorithm for any $f$-proper loss function. The algorithm is described in Figure \ref{algo:multi-group}.

It remains to construct the family of distinguishers $\A$, and to prove the reduction. Towards this, fix a group $g \in \G$ and fix a hypothesis $h \in \H$. We want to guarantee that the loss of the hypothesis $\tilde{h}(x) = f(x,\pt(x))$ is competitive with the loss of $h$, where both losses are measured on the distribution $\D_g$ over members of the group $g$. We begin by observing that this is true when the labels are drawn by $\pt(x)$ (as in the distribution $\tilde{\D}$). We will use OI (with an appropriately constructed distinguisher) to ensure that it is also true for the ``real'' distribution $\D_g$. 

In more detail, since $L$ is an $f$-proper loss function, we have: $L_{\tilde{\D}_g}(\tilde{h}) \leq L_{\tilde{\D}_g}(h)$,
because in $\tilde{\D}$ the labels are indeed generated by $\pt$, i.e. $\pt(x) = E_{\tilde{\D}} [y | x]$. By uniform convergence, this will remain true even if we consider the empirical loss over a (sufficiently large) i.i.d. sample from $\tilde{\D}_g$. With this in mind, we define a distinguisher $A^k_{g,h,\alpha}$, which takes as input $k$ samples $\{(x_i,y_i,p_i)\}$ and checks whether, for the samples where $x_i \in g$, it is true that the loss obtained by predicting $f(x_i,p_i)$ for each $x_i$ is competitive with the loss obtained by $h$ on those samples (up to an additive factor of $\alpha$). By the above discussion, when the labels $y_i$ values are drawn by $\text{Ber}(\pt(x_i))$, and assuming that there are sufficiently many samples in $g$ to guarantee uniform convergence for the loss $L$, the distinguisher will accept with high probability. See Figure \ref{algo:dist} for a full description of the distinguisher.

Now, if $\pt$ is OI w.r.t. a class that includes the distinguisher $A^k_{g,h,\alpha}$, then the distinguisher should accept with similar probabilities when the labeled examples are drawn by $\tilde{\D}$ or by $\D$ (where in both cases $p_i = \pt(x_i)$). I.e., $A^k_{g,h,\alpha}$ should also accept w.h.p. when the labeled examples are drawn by $\D$. This can only happen if the predictor $\tilde{h}$ is competitive with the hypothesis $h$ w.r.t. the distribution $\D_g$, which is exactly the guarantee we wanted from $\tilde{h}$!

The class $\A$ of distinguishers includes such a distinguisher for each $g \in \G$ and $h \in \H$, and thus if $\pt$ is OI w.r.t. $\A$, we conclude that the loss of $\tilde{h}$ is competitive with every $h \in \H$ for every group in $\G$ simultaneously. Note that the distinguishers in $\A$ use multiple samples, and the number of samples must be sufficiently large so that (for any sufficiently-large group $g$) w.h.p. enough of them fall in $g$ to guarantee uniform convergence. 

The sample complexity of the learning algorithm is governed by the sample complexity of OI learning, which is logarithmic in the number of distinguishers. Since the class $\A$ includes $|\G| \cdot |\H|$ distinguishers, the resulting learning algorithm has  sample complexity that is logarithmic in $|\G|$ and in $|\H|$. We note that we need $\G$ and $\H$ to be finite because the known OI learning algorithm works for finite collections of distinguishers.

\paragraph{Acknowledgements.} We thank Cynthia Dwork, Michael P. Kim and Omer Reingold for fruitful discussions throughout this work.

\bibliographystyle{apalike}
\bibliography{refs}

\begin{thebibliography}{}

\bibitem[Blum and Lykouris, 2019]{blum2019advancing}
Blum, A. and Lykouris, T. (2019).
\newblock Advancing subgroup fairness via sleeping experts.
\newblock {\em arXiv preprint arXiv:1909.08375}.

\bibitem[Blum and Stangl, 2019]{blum2019recovering}
Blum, A. and Stangl, K. (2019).
\newblock Recovering from biased data: Can fairness constraints improve
  accuracy?
\newblock {\em arXiv preprint arXiv:1912.01094}.

\bibitem[Buja et~al., 2005]{buja2005loss}
Buja, A., Stuetzle, W., and Shen, Y. (2005).
\newblock Loss functions for binary class probability estimation and
  classification: Structure and applications.
\newblock {\em Working draft, November}, 3.

\bibitem[Buolamwini and Gebru, 2018]{BuolamwiniG18}
Buolamwini, J. and Gebru, T. (2018).
\newblock Gender shades: Intersectional accuracy disparities in commercial
  gender classification.
\newblock In Friedler, S.~A. and Wilson, C., editors, {\em Conference on
  Fairness, Accountability and Transparency, {FAT} 2018, 23-24 February 2018,
  New York, NY, {USA}}, volume~81 of {\em Proceedings of Machine Learning
  Research}, pages 77--91. {PMLR}.

\bibitem[Chen et~al., 2018]{chen2018my}
Chen, I., Johansson, F.~D., and Sontag, D. (2018).
\newblock Why is my classifier discriminatory?
\newblock In {\em Advances in Neural Information Processing Systems}, pages
  3539--3550.

\bibitem[Chouldechova, 2017]{chouldechova2017fair}
Chouldechova, A. (2017).
\newblock Fair prediction with disparate impact: A study of bias in recidivism
  prediction instruments.
\newblock {\em Big data}, 5(2):153--163.

\bibitem[Dawid, 1982]{dawid1982well}
Dawid, A.~P. (1982).
\newblock The well-calibrated bayesian.
\newblock {\em Journal of the American Statistical Association},
  77(379):605--610.

\bibitem[Dwork et~al., 2012]{dwork2012fairness}
Dwork, C., Hardt, M., Pitassi, T., Reingold, O., and Zemel, R. (2012).
\newblock Fairness through awareness.
\newblock In {\em Proceedings of the 3rd innovations in theoretical computer
  science conference}, pages 214--226.

\bibitem[Dwork et~al., 2017]{dwork2017decoupled}
Dwork, C., Immorlica, N., Kalai, A.~T., and Leiserson, M. (2017).
\newblock Decoupled classifiers for fair and efficient machine learning.
\newblock {\em arXiv preprint arXiv:1707.06613}.

\bibitem[Dwork et~al., 2019]{dwork2019learning}
Dwork, C., Kim, M.~P., Reingold, O., Rothblum, G.~N., and Yona, G. (2019).
\newblock Learning from outcomes: Evidence-based rankings.
\newblock In {\em 2019 IEEE 60th Annual Symposium on Foundations of Computer
  Science (FOCS)}, pages 106--125. IEEE.

\bibitem[Dwork et~al., 2020]{dwork2020outcome}
Dwork, C., Kim, M.~P., Reingold, O., Rothblum, G.~N., and Yona, G. (2020).
\newblock Outcome indistinguishability.
\newblock {\em arXiv preprint arXiv:2011.13426}.

\bibitem[Foster and Vohra, 1998]{foster1998asymptotic}
Foster, D.~P. and Vohra, R.~V. (1998).
\newblock Asymptotic calibration.
\newblock {\em Biometrika}, 85(2):379--390.

\bibitem[Gupta et~al., 2021]{gupta2021online}
Gupta, V., Jung, C., Noarov, G., Pai, M.~M., and Roth, A. (2021).
\newblock Online multivalid learning: Means, moments, and prediction intervals.
\newblock {\em arXiv preprint arXiv:2101.01739}.

\bibitem[Hardt et~al., 2016]{hardt2016equality}
Hardt, M., Price, E., and Srebro, N. (2016).
\newblock Equality of opportunity in supervised learning.
\newblock In {\em Advances in neural information processing systems}, pages
  3315--3323.

\bibitem[{H{\'e}bert-Johnson} et~al., 2018]{hkrr}
{H{\'e}bert-Johnson}, {\'U}., Kim, M.~P., Reingold, O., and Rothblum, G.
  (2018).
\newblock Multicalibration: Calibration for the (computationally-identifiable)
  masses.
\newblock In {\em International Conference on Machine Learning}, pages
  1939--1948.

\bibitem[Jacobs et~al., 1991]{jacobs1991adaptive}
Jacobs, R.~A., Jordan, M.~I., Nowlan, S.~J., and Hinton, G.~E. (1991).
\newblock Adaptive mixtures of local experts.
\newblock {\em Neural computation}, 3(1):79--87.

\bibitem[Jung et~al., 2020]{jung2020moment}
Jung, C., Lee, C., Pai, M.~M., Roth, A., and Vohra, R. (2020).
\newblock Moment multicalibration for uncertainty estimation.
\newblock {\em arXiv preprint arXiv:2008.08037}.

\bibitem[Kearns et~al., 2018]{kearns2018preventing}
Kearns, M., Neel, S., Roth, A., and Wu, Z.~S. (2018).
\newblock Preventing fairness gerrymandering: Auditing and learning for
  subgroup fairness.
\newblock In {\em International Conference on Machine Learning}, pages
  2564--2572.

\bibitem[Kearns et~al., 1994]{kearns1994toward}
Kearns, M.~J., Schapire, R.~E., and Sellie, L.~M. (1994).
\newblock Toward efficient agnostic learning.
\newblock {\em Machine Learning}, 17(2-3):115--141.

\bibitem[Kim et~al., 2019]{kim2019multiaccuracy}
Kim, M.~P., Ghorbani, A., and Zou, J. (2019).
\newblock Multiaccuracy: Black-box post-processing for fairness in
  classification.
\newblock In {\em Proceedings of the 2019 AAAI/ACM Conference on AI, Ethics,
  and Society}, pages 247--254.

\bibitem[Kleinberg et~al., 2016]{kleinberg2016inherent}
Kleinberg, J., Mullainathan, S., and Raghavan, M. (2016).
\newblock Inherent trade-offs in the fair determination of risk scores.
\newblock {\em arXiv preprint arXiv:1609.05807}.

\bibitem[Kumar et~al., 2018]{kumar2018trainable}
Kumar, A., Sarawagi, S., and Jain, U. (2018).
\newblock Trainable calibration measures for neural networks from kernel mean
  embeddings.
\newblock In {\em International Conference on Machine Learning}, pages
  2805--2814. PMLR.

\bibitem[Rothblum and Yona, 2018]{yona2018probably}
Rothblum, G. and Yona, G. (2018).
\newblock Probably approximately metric-fair learning.
\newblock In {\em International Conference on Machine Learning}, pages
  5680--5688. PMLR.

\bibitem[Shabat et~al., 2020]{shabat2020sample}
Shabat, E., Cohen, L., and Mansour, Y. (2020).
\newblock Sample complexity of uniform convergence for multicalibration.
\newblock {\em arXiv preprint arXiv:2005.01757}.

\end{thebibliography}

\newpage
\appendix

\section{$L_\D(h) = \mathtt{FPR}_\D(h)$ doesn't have the uniform convergence property}
\label{appendix:fpr}

Recall that the  false positive rate of a binary classifier can be written as 

\begin{equation*}
    \mathtt{FPR}_\D(h) \equiv \Pr \sbr{h(x) = 1 \vert y = 0}
\end{equation*}

For intuition, the fact that we condition on the event $y=0$ means that for distributions in which the probability of this event is small, a good estimate of the true loss will require more samples. Intuitively, this contradicts the  uniform convergence requirement that there is a single number of samples that ``works'' for every distribution $\D$.

\begin{proof}
Suppose $\X$ is finite and $\card{\X}=n$. Fix some element $x' \in \X$, and consider the distribution $\D$ on $\X \times Y$ obtained by taking a uniform distribution over $\X$ and labeling elements via the deterministic labeling function
\begin{equation*}
    y(x)= \begin{cases}0 & x \neq x' \\ 1 & x=x' \end{cases}
\end{equation*}
Consider the classifier $h$ that labels everyone as 1: $h(x)\equiv 1$. Then, $\mathtt{FPR}_\D(h) = 1$ (since there is a single negative example, which is incorrectly labeled as a positive). On the other hand, the empirical estimate w.r.t any sample $S \subseteq \X - \set{x'}$ is zero, so for such a sample, the difference between $L_S(h)$ and $L_D(h)$ is at a maximal value of 1. Recall that uniform convergence requires us to estimate this difference to arbitrary precision with high probability; therefore, the ``bad event'' in which $x \notin S$ must happen with probability at most $\delta$. Equivalently, 
\begin{equation*}
    \br{\frac{n-1}{n}}^m \leq \delta
\end{equation*}

This require taking $m$ large enough to guarantee that $m \geq \log\frac{1}{\delta}\cdot\frac{1}{\log\frac{n}{n-1}}$. However, there is no function $f: (0,1) \to \N$ that guarantees that $m \geq f(\delta)$ satisfies this condition \emph{for every $n$}, because as $n$ approaches infinity, $\log \frac{n}{n-1} \to 0$, so the expression is unbounded.

\end{proof}

\section{Exist $L \in \Loss$ that are not multi-group compatible}
\label{appendix:acc-if}

\begin{proof}
For the counter-example we focus on binary classification and individual (metric) fairness w.r.t a binary metric (i.e., that specifies for every two individuals, whether they are identical or completely different). Fixing a metric $d: \X \times \X \to \set{0,1}$, the loss function $L$ is a combination of accuracy and individual (metric) fairness:

\begin{align*}
\label{eqn:loss-if-acc}
    L_\D(h) &= a \cdot L^{IF}_\D(h) + b\cdot L^{0-1}_\D(h) \\
    &\equiv a \cdot \Pr_{x,x'\sim \D_X}\sbr{h(x)\neq h(x') \land d(x,x')=0} + b\cdot \Pr_{x,y\sim \D} \sbr{h(x)\neq y} 
\end{align*}

Let us now construct the problem instance in question. Let the domain $\X$ be $\X = \set{x_S, x_T, x_{ST}}$, with $\D_X$ denoting the marginal distribution on $\X$ in which $x_S$ has probability 0.8, and $x_T, x_{ST}$ each have probability of 0.1. A distribution $\D$ is obtained
as the product of $\D_X$ and  $\D_{Y\vert X=x}$, where the latter assigns deterministic labels: $y(x_S) =0$ and $y(x_{ST}) = y(x_T)= 1 $. The class $\H$ consists of the constant classifiers, $h^0$ and $h^1$, and the collection of groups is $\G=\set{S,T}$ where $S =\set{x_S, x_{ST}}$ and $T=\set{x_T, x_{ST}}$. Finally, the metric specifies that $x_S$ and $x_{ST}$ are identical, and the rest are different:

\begin{equation*}
    d(x_S, x_{ST}) = 0, \quad,  d(x_T, x_{ST}) = 1 , \quad
    d(x_S, x_T) = 1
\end{equation*}

We argue that there is no classifier satisfying the multi-PAC requirement w.r.t $\H$ and $\G$. To see this, we first note that 

\begin{equation*}
    L_{\D_S}(\H) = b/9, \,\, L_{\D_T}(\H) = 0
\end{equation*}

This is because the optimal classifier for $T$ is $h^1$, which is perfect for both the accuracy and IF losses; whereas for $S$, the best classifier is $h^0$, which is perfect in terms of IF and has a 0-1 loss of $1/9$. 

Assume for contradiction that for every $\eps > 0$, there is a classifier that satisfies the multi-PAC requirement. 
From the perspective of $T$, the next-best classifier has a loss of $1/2$. So, for multi-PAC with $\eps < 1/2$, it must be the case that $h(x_{ST}) = 1$. On the other hand, from the perspective of $S$, when $a \gg b$ the next-best classifier has loss $8b/9$. So, for multi-PAC with $\eps < 7b/9$, it must be the case that $h(x_{ST})=0$. This means that for this problem instance and for $\eps < \min \set{7b/9, 1/2}$, there is no classifier satisfying the $\eps$-multi-PAC requirement.

\end{proof}

\section{Proof of Lemma 4.3 (compatibility $\to$ $f$-proper)}
\label{appendix:comp-to-proper}

Let $L$ be any unambiguous and compatible loss. 
First, we note that by unambiguity, for any singleton distribution $\D$, the loss minimizer is unique. We can therefore denote it by $h^\star_{\D}$.

Next, we make an observation that we will use in the proof: that for any distribution $\D$, the classifier $h: \X \to [0,1]$ defined by

\begin{equation*}
    h(x) = 
    \begin{cases}
    h^\star_{\D_x}(x), & x \in \supp(\D) \\
    0, & \text{otherwise}
    \end{cases}
\end{equation*}

minimizes the loss $L_\D(\cdot)$. That is, we are forming a new classifier   $h$
 by predicting on an input $x \in \supp(\D)$ using the prediction of the classifier that minimizes the loss on the distribution $\D$ restricted to $x$. The claim is that this classifier is competitive with the best possible loss on the original distribution $\D$.
 
  We claim that the observation follows as a corollary from the compatibility assumption. Note that this is trivially the case for any singleton distribution $\D$ (by unambiguity), so assume for contradiction that there is a non-singleton distribution $\D$ for which the observation does not hold. We define a multi-PAC problem instance as follows. For $x \in \X$, let $g_x = \set{x}$ denote the singleton group that consists only of $x$. Define $\G_{singletons} = \set{g_x: \,\, x \in \supp(\D)}$ and $\H_{singletons} =\set{h^\star_{\D_x}: \,\, x \in \supp(\D)} $. Additionally, let $h^\star$ denote some classifier in $\arg\min_h L_\D(h)$. Finally, 
  
  \begin{align*}
      \G &= \G_{singletons} \cup \set{\supp(\D)} \\
      \H &= \H_{singletons} \cup \set{h^\star}
  \end{align*}
  
  Note that by definition, for every group $g$ in $\G$, $L_{\D_g}(\H) =0$ (because we specifically included an optimal classifier for every group in $\H$). Multi-PAC for the singleton groups $\G_{singeltons}$ with an arbitrarily small precision $\eps$ therefore requires that we predict $h^\star_{\D_x}(x)$ for $x \in \supp(\D)$. But by the assumption, the resulting classifier is not optimal for the group $\set{\supp(\D)}$, in violation to multi-PAC w.r.t that group. This completes the proof of the observation.

  We can now use the observation to directly prove niceness. We will do this by constructing a specific function $f$ and showing that the classifier defined by $h_\D(x) = f(x, \ygivenx)$ always minimizes the loss $L_\D$. Consider

 \begin{equation*}
     f(x,v) = h^\star_{\D_{x,v}}(x)
 \end{equation*}
 
 where $\D_{x,v}$ is the singleton distribution supported  on $x$ that predicts a label of 1 w.p $v$, and $ h^\star_{\D_{x,v}}$ is the loss minimizer for this distribution (which, by unambiguity, is indeed unique). 
 
 We need to prove that $f$ satisfies the requirement in the definition of $f-$proper. Fix some distribution $\D$; we need to prove that 
 
 \begin{equation*}
     h_\D \in \arg\min_h L_\D(h)
 \end{equation*}
 where $h_\D(x) = f(x, \ygivenx)$.

 By the observation, the classifier that predicts for $x \in \supp(\D)$ using the optimal classifier for $\D_x$ is itself optimal for $\D$. But by construction, $ \D_x \equiv  \D_{x,\ygivenx}$, so we get:
 
   \begin{equation*}
       h^\star_{\D_x}(x) = h^\star_{\D_{x, \ygivenx}} =
       f(x,\ygivenx) = h_\D(x)
  \end{equation*}
  
  In other words, the classifier that predicts for  $x \in \supp(\D)$ using the optimal classifier for $\D_x$ is precisely $h_\D$. The observation therefore proves $h_\D \in \arg\min_h L_\D(h)$, as required.

\section{Proof of Lemma 4.4 ($f$-proper $\to$ learnability)}
\label{appendix:proper-to-learnable}

To prove the lemma, we construct a learning algorithm, $\mathtt{MultiGroup_L}$, and prove that when $L$ is compatible and has the uniform convergence property, the output of this algorithm satisfies the requirements in the definition of multi-group learnability.

The definition of $\mathtt{MultiGroup_L}$ is given in Algorithm \ref{algo:multi-group}. At a high-level, $\mathtt{MultiGroup_L}$ accepts a class $\H$, collection of subgroups $\G$ and parameters $\eps, \delta, \gamma$, and returns a classifier by invoking a learning algorithm for OI w.r.t an appropriate distinguisher class $\A$. The definition of each distinguisher $A \in \A$ is given separately -- see Algorithm \ref{algo:dist}.

\begin{algorithm}[h]
  \caption{$\mathtt{MultiGroup_{L,f}}(\epsilon,\delta,\gamma,\H,\G)$}
  \label{algo:multi-group}
      \begin{algorithmic}[1]
        \STATE {\bfseries Parameters:} loss function $L$, function $f: \X \times [0,1] \to [0,1]$
        
        \vspace{2mm}
        \STATE {\bfseries Input:} accuracy parameter $\epsilon \in (0,1)$, failure probability $\delta \in (0,1)$, minimal subgroups size parameter $\gamma \in (0,1)$, hypothesis class $\H$, collection of subgroups $\G$. 
        \STATE {\bfseries Output:} A classifier $h$ satisfying the $(\eps, \delta)$-multi-group guarantee w.r.t $\H$ and $\G$
        \vspace{2mm}

        \STATE Set $\eps'=\alpha = \eps/4$ and $\delta'=\eta = \tau = \delta/4$.

        \STATE{Set $k_\G = m^{UC}_L\br{\eps', \delta', \cH + 1}$}.
        
        \STATE Set $ k=10\cdot\frac{1}{\gamma}\cdot\log\frac{1}{\delta'}\cdot k_\G $.
        
        \STATE Let $\A = \set{A^{L, f, k}_{ g,h,\alpha} \,\, \vert \,\, g\in \G, h\in \H} $ be a collection of distinguishers, as defined in Algorithm \ref{algo:dist}.

        \STATE Invoke OI as a sub-routine to learn $\tilde{p}\leftarrow\mathtt{OI}(\tau,\eta,\A)$.
        
        \RETURN{$f(\tilde{p})$}
        
      \end{algorithmic}
\end{algorithm}

\begin{algorithm}[h]
  \caption{$A^{L, f, k}_{ g,h,\alpha}$ (multi-sample Sample-Access OI distinguisher)}
  \label{algo:dist}
      \begin{algorithmic}[1]
        \STATE {\bfseries Parameters:} number of samples $k \in \N$, group $g \subseteq \X$, classifier $h: \X \to [0,1]$, loss function $L$, function $f: \X \times [0,1] \to [0,1]$
        \vspace{2mm}
        \STATE {\bfseries Input:} $\set{(x_i, y_i, p_i)}_{i=1}^k$, where $x_i \in \X$, $y_i \in \set{0,1}$ and $p_i \in [0,1]$
        \STATE {\bfseries Output:} A binary output denoting Accept/Reject
        \vspace{2mm}
        
        \STATE $I_g = \set{i: x_i \in g}$ 
        \STATE $S_g = \set{(x_i, y_i)}_{i \in I_g}$
        
        \STATE Define a predictor $f_g$ as
        \begin{equation}
        \label{eqn:f_g}
        f_g(x) = 
        \begin{cases}
         f(x_i, p_i) & \exists i \in [k] \,\, \text{such that} \,\,  x=x_i \\
         0 & \text{otherwise}
        \end{cases}
        \end{equation}
        
        \IF{$L_{S_g}(f_g) < L_{S_g}(h) + 2\alpha$}
            \RETURN{1}
        \ENDIF
        \RETURN{0}

      \end{algorithmic}
\end{algorithm}

We begin by proving that if 
$L$ is $f$-proper, then $h\leftarrow\mathtt{MultiGroup_{L,f}}(\eps,\delta,\H,\G)$ satisfies the $(\eps,\delta)$-multi-group requirement w.r.t $\H$ and $\G$.

\begin{lemma}
\label{lemma:multi-group-is-pac}
Suppose $L$ is $f$-proper. Fix a distribution $\D$ over $\X\times Y$, a finite class $\H$, a finite collection of subgroups $\G$ and parameters $\delta, \eps, \gamma \in (0,1)$. Then,
w.p at least $1-\delta$, the predictor $h \leftarrow \mathtt{MultiGroup_{L,f}}(\epsilon,\delta,\gamma, \H,\G)$ satisfies

\begin{equation*}
       \forall g \in \G_\gamma: \quad  L_{\D_g}(h) \leq L_{\D_g}(\H) + \eps
\end{equation*}

\end{lemma}

\begin{proof}
We begin by lower-bounding the acceptance probability of each distinguisher $A \in \A$ when it receives samples from the modeled distribution $\Dt$. Recall that this is the distribution in which outcomes $y_i$ are sampled according to $\text{Ber}(\pt(x_i))$, where $\pt$ is the predictor returned by OI.

\begin{claim}
 The probability that each $A \triangleq  A_{g,h} \in \A$ accepts when given samples from the modeled distribution $\Dt$ is at least $1-2\delta'$: 
 \begin{equation*}
     \Pr_{\{(x_{i},y_{i})\}_{i=1}^{k}\sim\tilde{D}^{k}}[A(\{(x_{i},y_{i},\tilde{p}(x_{i})\}_{i=1}^{k})=1]\geq1-2\delta'
 \end{equation*}
\end{claim}

To see why this is true, we first note that by construction, the predictor  $f(\tilde{p})$  (where $f(\tilde{p})(x)=f(x,\tilde{p}(x))$) coincides with the predictor $h_{\Dt}$ from the definition of $f-$proper. Thus, invoking the assumption that  $L$ is $f$-proper for the distribution $\D_g$  guarantees that  
\begin{equation}
\label{eqn:from-niceness}
    L_{\Dt_g}f(\tilde{p})\leq L_{\Dt_g}(h)
\end{equation}

To relate this fact to the acceptance criteria of $A$, which is in terms of a sample $S_g$ from $\Dt_g$, we need to use the uniform convergence property of $L$. Recall that the distinguisher operates on $k=10\cdot\frac{1}{\gamma}\cdot\log\frac{1}{\delta'}\cdot k_\G$ samples from $\Dt$; this was chosen precisely to guarantee that w.p at least $1-\delta'$, we have at least $k_\G$ samples from $\Dt_g$ for every group $g$ whose mass is at least $\gamma$. Since $k_\G = m^{UC}_L\br{\eps', \delta', \cH + 1}$, we have a uniform convergence guarantee for the class that includes $\H$ and $\pt$. That is, we know that w.p at least $1-\delta'$

\begin{align}
\label{eqn:uc_h_pt}
    \left|L_{S_{g}}(h)-L_{\tilde{\D}_{g}}(h)\right|\leq\eps', \qquad 
    \left|L_{S_{g}}(f(\tilde{p}))-L_{\tilde{\D}_{g}}(f(\tilde{p}))\right|\leq\eps'
\end{align}

Combining Equations (\ref{eqn:from-niceness}) and (\ref{eqn:uc_h_pt}), we have that with probability at least $1-2\delta'$ (obtained by union bounding with respect to the two $\delta'$ failure probabilities we used above),

\begin{equation*}
    L_{S_{g}}(f(\tilde{p})) \leq L_{S_{g}}(h)+2\epsilon'
\end{equation*}

Finally, we note that w.r.t $S_g$, the predictor $f_g$ defined in Equation (\ref{eqn:f_g}) of Algorithm \ref{algo:dist} is the same as $f(\pt)$ -- so the above is precisely the acceptance criterion for $A$ in this case. We thus conclude that the acceptance probability of $A$ when it receives samples from the modeled distribution is at least $1-2\delta'$, which concludes the proof of the claim.

Next, we argue that a direct corollary is a related lower bound on the acceptance probability of $A$ when it receives samples from the true distribution $\D$.

\begin{claim}
 The probability that each $A \triangleq  A_{g,h} \in \A$ accepts when given samples from the true distribution $\D$ is at least $1-(2\delta' + \tau + \eta)$: 
 \begin{equation*}
     \Pr_{\{(x_{i},y_{i})\}_{i=1}^{k}\sim\D^{k}}[A(\{(x_{i},y_{i},\tilde{p}(x_{i})\}_{i=1}^{k})=1]\geq 1-(2\delta' + \tau + \eta)
 \end{equation*}
\end{claim}

The claim follows as a direct corollary from the previous claim. By definition, since $\mathtt{OI}$ is a learning algorithm for OI predictors, $\pt$ is $(\tau, \A)$-OI w.p at least $1-\eta$. Recall that if $\pt$ is $(\tau, \A)$-OI, then we are guaranteed that the  probabilities of each $A \in \A$ accepting on samples from $\Dt$ and $A$ accepting on samples from $\D$ differ by at most $\tau$:

\begin{equation}
\label{eqn:oi-D}
    \left|\Pr_{\{(x_{i},y_{i})\}_{i=1}^{k}\sim \D^{k}}[A(\{(x_{i},y_{i},\tilde{p}(x_{i})\}_{i=1}^{k})=1]-\Pr_{\{(x_{i},y_{i})\}_{i=1}^{k}\sim\tilde{\D}^{k}}[A(\{(x_{i},y_{i},\tilde{p}(x_{i})\}_{i=1}^{k})=1]\right|\leq\tau
\end{equation}

in other words, w.p at least $1-\eta$ we are guaranteed that $\Pr_{\{(x_{i},y_{i})\}_{i=1}^{k}\sim \D^{k}}[A(\{(x_{i},y_{i},\tilde{p}(x_{i})\}_{i=1}^{k})=1] \geq 102\delta'-\tau$. This implies that a lower bound on the acceptance probability in this case is exactly $ 1-(2\delta' + \tau + \eta)$, completing the proof of the claim.

Next, we recall that by the definition of the acceptance condition for $A$, the condition in Equation (\ref{eqn:oi-D}) is the same as saying that w.p at least $ 1-(2\delta' + \tau + \eta)$ over the choice of $S_g \sim \D_g$, 

\begin{equation*}
    L_{S_{g}}(f(\tilde{p}))\leq L_{S_{g}}(h)+2\eps'
\end{equation*}

Again using the uniform convergence guarantee from Equation (\ref{eqn:uc_h_pt}), this implies

\begin{equation*}
    L_{\D_{g}}(f(\tilde{p}))\leq L_{\D_{g}}(h)+4\eps'
\end{equation*}

Plugging in $\eps' = \eps/4$ and $\delta'=\tau=\eta=\delta/4$, we conclude that w.p at least $1-\delta$,

\begin{equation*}
    L_{\D_{g}}(f(\tilde{p}))\leq L_{\D_{g}}(h)+\eps
\end{equation*}

which is the required. This completes the proof of Lemma \ref{lemma:multi-group-is-pac}.
\end{proof}

To prove multi-group learnability, it remains to bound the sample complexity of Algorithm \ref{algo:multi-group}, which we do in the following claim.

\begin{claim}
The sample complexity of Algorithm \ref{algo:multi-group} is 

\begin{equation*}
    m^{gPAC}_{L}(\eps, \delta, \gamma, \H, \G) = O\left(\frac{m_{\H}(\eps,\delta)\cdot\log \br{\frac{\cH \cdot \cG}{\eps}}}{\delta^{4}\cdot \gamma}\right)
\end{equation*}
\end{claim}

\begin{proof}
By the definition of Algorithm \ref{algo:multi-group}, the number of samples required is the number of samples required to obtain OI w.r.t $(\tau, \eta, \A)$, where $\A$ is a collection of $\card{\H}\cdot \card{\G}$ $k-$sample OI distinguishers. By Theorem 2.8, this requires an order of $O(\frac{k \cdot \log(\card{A}/\eta)}{\tau^4})$ samples. Ignoring constant factors and plugging in the settings of $k, \eta$ and $\tau$ used in Algorithm \ref{algo:multi-group},

\begin{align*}
    \eta &= O(\delta) \\
    \tau &= O(\eps) \\
    k  &= O\br{\frac{1}{\gamma} \cdot \log \frac{1}{\delta} \cdot m^{UC}_L(\eps, \delta, \cH)} = O\br{\frac{1}{\gamma} \cdot \log \frac{1}{\delta} \cdot m_\H(\eps, \delta} 
\end{align*}
we obtain the stated bound.
\end{proof}

Note that when $L$ has the uniform convergence property, this entire expression is indeed polynomial in $1/\eps, 1/\delta, 1/\gamma$ and $\log(\cH), \log(\cG)$, as required. Together with the previous lemma, this concludes the proof of Lemma 4.4.

\end{document}